\documentclass[11pt]{article}

\usepackage[colorlinks=true, citecolor=teal, linkcolor=black]{hyperref}

\usepackage{vitercik}
\usepackage{caption, subcaption}

\newcommand{\cost}{\texttt{cost}}
\newcommand{\score}{\texttt{score}}
\newcommand{\ascore}{\texttt{ascore}}
\newcommand{\nscore}{\texttt{nscore}}
\newcommand{\tree}{\pazocal{T}}
\newcommand{\node}{Q}
\newcommand{\LP}{\mathsf{LP}}

\newcommand{\children}{\texttt{children}}
\newcommand{\actions}{\texttt{actions}}
\newcommand{\fathom}{\texttt{fathom}}

\title{Improved Sample Complexity Bounds for Branch-and-Cut}
\author{Maria-Florina Balcan\thanks{School of Computer Science, Carnegie Mellon University. \texttt{ninamf@cs.cmu.edu}} \and Siddharth Prasad\thanks{Computer Science Department, Carnegie Mellon University. \texttt{sprasad2@cs.cmu.edu}} \and Tuomas Sandholm\thanks{Computer Science Department, Carnegie Mellon University, Optimized Markets, Inc., Strategic Machine, Inc., Strategy Robot, Inc. \texttt{sandholm@cs.cmu.edu}} \and Ellen Vitercik\thanks{Department of Electrical Engineering and Computer Sciences, UC Berkeley. \texttt{vitercik@berkeley.edu}}}

\begin{document}

\maketitle

\begin{abstract}
Branch-and-cut is the most widely used algorithm for solving integer programs, employed by commercial solvers like CPLEX and Gurobi. Branch-and-cut has a wide variety of tunable parameters that have a huge impact on the size of the search tree that it builds, but are challenging to tune by hand. An increasingly popular approach is to use machine learning to tune these parameters: using a \emph{training set} of integer programs from the application domain at hand, the goal is to find a configuration with strong predicted performance on future, unseen integer programs from the same domain. If the training set is too small, a configuration may have good performance over the training set but poor performance on future integer programs. In this paper, we prove \emph{sample complexity guarantees} for this procedure, which bound how large the training set should be to ensure that for any configuration, its average performance over the training set is close to its expected future performance. Our guarantees apply to parameters that control the most important aspects of branch-and-cut: node selection, branching constraint selection, and cutting plane selection, and are sharper and more general than those found in prior research~\citep{Balcan18:Learning,Balcan21:Sample}.
\end{abstract}

\section{Introduction}
Branch-and-cut (B\&C) is a powerful algorithmic paradigm that is the backbone of all modern integer programming (IP) solvers. The main components of B\&C can be tuned and tweaked in a myriad of ways. The fastest commercial IP solvers like CPLEX and Gurobi employ an array of heuristics to make decisions at every stage of B\&C to reduce the solving time as much as possible, and give the user freedom to tune the multitude of parameters influencing the search through the space of feasible solutions. However, tuning the parameters that control B\&C in a principled way is an inexact science with little to no formal mathematical guidelines. A rapidly growing line of work studies machine-learning approaches to speeding up the various aspects of B\&C---in particular investigating whether high-performing B\&C parameter configurations can be learned from a \emph{training set} of typical IPs from the particular application at hand~\cite{Alvarez17:Machine,Horvitz01:Bayesian,Sandholm13:Very-Large-Scale,Xu08:SATzilla,Hutter09:Paramils,Leyton-Brown09:Empirical,Kadioglu10:ISAC,Xu11:Hydra-MIP,Khalil16:Learning}. Complementing the substantial number of experimental approaches using machine learning for B\&C, a recent generalization theory has developed in parallel that aims to provide a rigorous theoretical foundation for how well any B\&C configuration learned from training IP data will perform on new unseen IPs~\cite{Balcan18:Learning,Balcan21:Sample}.
In particular, this line of theoretical research provides \emph{sample complexity guarantees} that bound how large the training set should be to ensure that \emph{no matter how the parameters are configured} (i.e., using any approach from prior research), the average performance of branch-and-cut over the training set is close to its expected future performance. Sample complexity bounds are important because with too small a training set, learning is impossible: a configuration may have strong average performance over the training set but terrible expected performance on future IPs. If the training set is too small, then no matter how the parameters are tuned, the resulting configuration will not have reliably better performance than any default configuration. State-of-the-art parameter tuning methods have historically come without any provable guarantees, and our results fill in that gap for a wide array of tunable B\&C parameters.
 In this paper, we expand and improve upon the existing theory to develop a wider and sharper handle on the learnability of the key components of B\&C. 

\subsection{Summary of main contributions}\label{sec:contrib}

Our main contribution is a formalization of a general model of tree search, presented in Section~\ref{sec:model}, that allows us to improve and generalize prior results on the sample complexity of tuning B\&C. In this model, the algorithm repeatedly chooses a leaf node of the search tree, performs a series of actions (for example, a cutting plane to apply and a constraint to branch on), and adds children to that leaf in the search tree. The algorithm will also fathom nodes when applicable. The node and action selection are governed by \emph{scoring rules}, which assign a real-valued score to each node and possible action. For example, a node-selection scoring rule might equal the objective value of the node's LP relaxation. We focus on general tree search with \emph{path-wise} scoring rules. At a high level, a score of a node or action is path-wise if its value only depends on information contained along the path between the root and that node, as is often the case in B\&C. Many commonly used scoring rules are path-wise including the  efficacy~\cite{Balas96:Mixed}, objective parallelism~\cite{Achterberg07:Constraint}, directed cutoff distance~\cite{Gamrath20:SCIP}, and integral support~\cite{Wesselmann12:Implementing} scoring rules, all used for cut selection by the leading open-souce solver SCIP~\cite{Gamrath20:SCIP}; the best-bound scoring rule for node selection; and the linear, product, and most-fractional scoring rules for variable selection using strong branching~\cite{Achterberg07:Constraint}. In Section~\ref{sec:branch_and_cut}, we show how this general model of tree search captures a wide array of B\&C components, including node selection, general branching constraint selection, and cutting plane selection, simultaneously. We also provide experimental evidence that, in the case of cutting plane selection, the data-dependent tuning suggested by our model can lead to dramatic reductions in the number of nodes expanded by B\&C.

In Section~\ref{sec:tree}, we prove our main structural result: for any IP, the tree search parameter space can be partitioned into a finite number of regions such that in any one region, the resulting search tree is fixed. This is in spite of the fact that the B\&C search tree can be an extremely unstable function of its parameters, with minuscule changes leading to exponentially better or worse performance~\cite{Balcan18:Learning,Balcan21:Sample}. By analyzing the complexity of this partition, we prove our sample complexity bound. In particular, we relate the complexity of the partition to the \emph{pseudo-dimension} of the set of functions that measure the performance of B\&C as a function of the input IP. Pseudo-dimension (defined in Section~\ref{sec:tree}) is a combinatorial notion from machine learning theory that measures the \emph{intrinsic complexity} of a set of functions. At a high level, it measures how well a set of functions are able match complex patterns. Classic results from learning theory then allow us to translate our pseudo-dimension bound into a sample complexity guarantee~\cite{Anthony09:Neural}, capturing the intuition that the more complex patterns one can fit (i.e., the larger the pseudo-dimension is), the more samples needed to generalize. The sample complexity bound grows linearly with the pseudo-dimension, so ideally, the pseudo-dimension will be polynomial in the size of the problem.


We show that the pseudo-dimension is only polynomial in the depth of the tree (which is, for example, at most the number of variables in the case of binary integer programming). By contrast, we might na\"ively expect the pseudo-dimension to grow linearly with the number of arithmetic operations required to compute the B\&C tree (as in Theorem 8.4 by Anthony and Bartlett~\cite{Anthony09:Neural}), which is exponential in the depth of the tree.
In fact, our bound is exponentially smaller than the pseudo-dimension bound of prior research by Balcan et al.~\cite{Balcan21:Sample}, which grows linearly with the total number of nodes in the tree. Their results apply to any type of scoring rule, path-wise or otherwise. By taking advantage of the path-wise structure, we are able to reason inductively over the depth of the tree, leading to our exponentially improved bound.
Our results recover those of Balcan et al.~\cite{Balcan18:Learning}, who only studied path-wise scoring rules for single-variable selection for branching. In contrast, we are able to handle many more of the critical components of tree search: node selection, general branching constraint selection, and cutting plane selection.

\subsection{Additional related research}

A growing body of research has studied how machine learning can be used to speed up the time it takes to solve integer programs, primarily from an empirical perspective, whereas we study this problem from a theoretical perspective. This line of research has included general parameter tuning procedures~\cite{Hutter09:Paramils,Kadioglu10:ISAC,Hutter11:Sequential,Sandholm13:Very-Large-Scale}, which are not restricted to any one aspect of B\&C. Researchers have also honed in on specific aspects of tree search and worked towards improving those using machine learning. These include variable selection~\cite{Khalil16:Learning,Alvarez17:Machine,DiLiberto16:Dash, Balcan18:Learning,Gasse19:Exact,Gupta20:Hybrid}, general branching constraint selection~\cite{Yang20:Learning}, cut selection~\cite{Sandholm13:Very-Large-Scale, Tang20:Reinforcement, Huang21:Learning, Balcan21:Sample}, node selection~\cite{Sabharwal12:Guiding,He14:Learning}, and heuristic scheduling~\cite{Khalil17:Learning,Chmiela21:Learning}. Machine learning approaches to large neighborhood search have also been used to speed up solver runtimes~\cite{Song20:General}.

This paper contributes to a line of research that provides sample complexity guarantees for algorithm configuration, often by using structure exhibited by the algorithm's performance as a function of its parameters~\cite{Gupta17:PAC,Balcan17:Learning,Balcan18:Learning,Balcan21:How,Balcan21:Sample, Balcan20:Data}. This line of research has studied algorithms for clustering~\cite{Balcan17:Learning}, computational biology~\cite{Balcan21:How}, and integer programming~\cite{Balcan18:Learning,Balcan21:Sample}, among other computational problems. The main contribution of this paper is to provide a sharp yet  general analysis of the performance of tree search as a function of its parameters.

A related line of research provides algorithm configuration procedures with provable guarantees that are agnostic to the specific algorithm that is being configured~\cite{Kleinberg17:Efficiency,Weisz18:LeapsAndBounds} and are particularly well-suited for algorithms with a finite number of possible configurations (though they can be applied to algorithms with infinite parameter spaces by randomly sampling a finite set of configurations).

\section{Main tree search model}

In this section we present our general tree search model and situate it within the framework of sample complexity. Balcan et al.~\cite{Balcan21:Sample} studied the sample complexity of a much more general formulation of a tunable search algorithm without any inherent tree structure. Our formulation explicitly builds a tree. 

\subsection{General model of tree search}\label{sec:model}

Tree search starts with a root node. In each round of tree search, a leaf node $Q$ is selected. At this node, one of three things may occur: (1) $Q$ is fathomed, meaning it is never visited again, (2) some action is taken at $Q$, and then it is fathomed, or (3) some action is taken at $Q$, and then some number of children nodes of $Q$ are added to the tree. (For example, an action might represent a decision about which variable to branch on.) This process repeats until the tree has no unfathomed leaves. More formally, there are functions $\actions$, $\children$, and $\fathom$ prescribing how the search proceeds. Given a partial tree $\tree$ and a leaf $Q$ of $\tree$, $\actions(\tree, Q)$ outputs a set of actions available at $Q$. Given a partial tree $\tree$, a leaf $Q$ of $\tree$, and an action $A\in\actions(\tree, Q)$, $\fathom(\tree, \node, A)\in\{\texttt{true}, \texttt{false}\}$ is a Boolean function used to determine when to fathom a leaf $Q$ of $\tree$ given that action $A\in\actions(\tree, \node)\cup\{\texttt{None}\}$ was taken at $Q$, and $\children(\tree, Q, A)$ outputs a (potentially empty) list of nodes representing the children of $Q$ to be added to the search tree given that action $A$ was taken at $Q$.
Finally, $\nscore(\tree, Q)$ is a node-selection score that outputs a real-valued score for each leaf of $\tree$, and $\ascore(\tree, Q, A)$ is an action-selection score that outputs a real-valued score for each action $A\in\actions(\tree, Q)$. These scores are heuristics that are meant to indicate the quality of exploring a node or performing an action.

Many aspects of B\&C are governed by scoring rules~\cite{Achterberg07:Constraint}. For example, commonly used scoring rules for cutting plance selection include \emph{efficacy}~\cite{Balas96:Mixed}, which is the perpendicular distance from the current LP solution to the cutting plane; \emph{parallelism}~\cite{Achterberg07:Constraint}, which measures the angle between the objective and the normal vector to the cutting plane; and \emph{directed cutoff}~\cite{Gamrath20:SCIP}, which is the distance from the current LP solution to the cutting plane along the direction of the line segment connecting the LP solution to the current best incumbent integer solution
For node selection, under the commonly used best-first node selection policy, $\nscore(\tree, Q)$ equals the objective value of the LP relaxation of the IP represented by the node $Q$.
Finally, for variable selection, popular scoring rules include a maximum change in LP objective value after branching on the variable (where the maximum is taken over the two resulting children), the minimum change in the LP objective value, linear combinations of these two values, and the product of these two values~\cite{Achterberg07:Constraint}. Algorithm~\ref{alg:TS} is a formal description of tree search using these functions.

\begin{algorithm}[t]
	\caption{Tree search}\label{alg:TS}
	\begin{algorithmic}[1]
\Require Root node $\node$, depth limit $\Delta$
\State Initialize $\tree = \node$.
\While{$\tree$ contains an unfathomed leaf\label{step:while_begin}}
	\State Select a leaf $Q$ of $\tree$ that maximizes $\nscore(\tree, Q)$.\label{step:nsp}
	\If {$\texttt{depth}(\node) = \Delta$ or $\fathom(\tree, \node, \texttt{None})$} \label{step:depth_fathom}
	    \State Fathom $Q$.
	\Else
	    \State Select an action $A\in\actions(\tree, \node)$ that maximizes $\ascore(\tree,Q, A)$.\label{step:action}
	    \If {$\fathom(\tree, \node, A)$} \label{step:action_fathom}
	        \State Fathom $Q$.
	    \ElsIf {$\children(\tree, \node, A) = \emptyset$} \label{step:no_children}
	        \State Fathom $Q$.
	    \Else 
	        \State Add all nodes in $\children(\tree, \node, A)$ to $\tree$ as children of $Q$. \label{step:add_children}
	    \EndIf
    \EndIf
\EndWhile\label{step:while_end}
\end{algorithmic}
\end{algorithm}

The key condition that enables us to derive stronger sample complexity bounds compared to prior research is the notion of a \emph{path-wise} function, which was also used in prior research but only in the context of variable selection~\cite{Balcan18:Learning}.

\begin{definition}[Path-wise functions]
A function $f$ on tree-leaf pairs is path-wise if for all $\tree$ and $Q\in\tree$, $f(\tree, \node) = f(\tree_Q, \node)$, where $\tree_Q$ is the path from the root of $\tree$ to $Q$. A function $g$ on tree-leaf-action triples is path-wise if for all $A$, $f_A(\tree, Q) := g(\tree, Q, A)$ is path-wise.
\end{definition}

We assume that $\actions$, $\ascore$, $\nscore$ and $\children$ are path-wise, though $\fathom$ is not necessarily path-wise. 

Many commonly-used scoring rules are path-wise. For example, scoring rules are often functions of the LP relaxation of the IP represented by a given node, and these scoring rules are path-wise. Specific examples include the efficacy, objective parallelism, directed cutoff distance, and integral support scoring rules used for cut selection; the best-bound scoring rule for node selection; and the linear, product, and most-fractional scoring rules for variable selection using strong branching. A point of clarification: the pathwise assumption is with respect to the numerical scores assigned to actions/nodes. The actual act of, for example, node selection, can depend on the entire tree. For example, consider the best-bound node selection rule in branch-and-cut, which chooses the node with the best LP estimate. Here, the scoring rule, which is the LP objective value itself, is pathwise, but ultimately the node that is selected depends on the LP bounds at every unexplored node of the tree. This is fine for our analysis. Similarly, the decision to fathom a node based on LP bounds is a decision that depends on the entire tree built so far, which is also captured by our analysis.

No one scoring rule is optimal across all application domains, and prior research on variable selection has shown that it can be advantageous to adapt the scoring rule to the application domain at hand~\cite{Balcan18:Learning}.
To this end, Algorithm~\ref{alg:TS} can be tuned by two parameters $\mu\in [0,1]$ and $\lambda\in [0,1]$ that control action selection and node selection, respectively. Given two fixed path-wise action-selection scores $\ascore_1$ and $\ascore_2$, we define a new score by $$\ascore_{\mu}(\tree, Q) = \mu\cdot\ascore_1(\tree, Q) + (1-\mu)\cdot\ascore_2(\tree, Q).$$ Similarly, given two path-wise node-selection scores $\nscore_1$ and $\nscore_2$, we define  $$\nscore_{\lambda}(\tree, Q, A) = \lambda\cdot\nscore_1(\tree, Q, A) + (1-\lambda)\cdot\nscore_2(\tree, Q, A).$$ Then, if $\nscore_{\lambda}$ and $\ascore_{\mu}$ are used as the scores in Algorithm~\ref{alg:TS}, we can view the behavior of tree search as a function of $\mu$ and $\lambda$. The choice to use a convex combination of scores is not new: prior research has shown that this idea can lead to dramatic improvements in the case of single-variable branching~\cite{Balcan18:Learning}. Furthermore, the leading open source solver SCIP uses a hard-coded weighted sum of scoring rules to select cutting planes. More broadly, interpolating between two scores is a commonly-studied modeling choice in other machine learning topics such as clustering~\cite{Balcan17:Learning}.

Finally, we assume there exists $b, k\in\N$ such that $|\actions(\tree, \node)|\le b$ for any $\node\in\tree$, and $|\children(\tree, \node, A)|\le k$ for all $Q, A$.

\subsection{Problem formulation}\label{sec:learning}

We now define the notion of a \emph{sample complexity bound} more formally.
Let $\cQ$ denote the domain of possible input root nodes $Q$ to Algorithm~\ref{alg:TS} (for example, the set of all IPs with $n$ variables and $m$ constraints). As is common in prior research on algorithm configuration~\cite{Horvitz01:Bayesian,Sandholm13:Very-Large-Scale,Xu08:SATzilla,Hutter09:Paramils,Leyton-Brown09:Empirical,Kadioglu10:ISAC,Xu11:Hydra-MIP}, we assume there is some unknown distribution $\dist$ over $\cQ$. In the IP setting, $\dist$ could represent, for example, typical scheduling IP instances solved by an airline company. 
The \emph{sample complexity} of a class of real valued functions $\cF = \{f : \cQ\to\R\}$ is the minimum number of independent samples required from $\dist$ so that with high probability over the samples, the empirical value of $f$ on the samples is a good approximation of the expected value of $f$ over $\dist$, uniformly over all $f\in\cF$. Formally, given an error parameter $\varepsilon$ and confidence parameter $\delta$, the sample complexity $N_{\cF}(\varepsilon, \delta)$ is the minimum $N_0\in\N$ such that for any $N\ge N_0$, $${\displaystyle\Pr_{Q_1,\ldots, Q_N\sim\dist}\left(\sup_{f\in\cF}\left|\frac{1}{N}\sum_{i=1}^N f(Q_i) - \E_{Q\sim\dist}[f(Q)]\right|\le\varepsilon\right)\ge 1-\delta}$$ for all distributions $\dist$ supported on $\cQ$. Equivalently, our results bound the error $\varepsilon_{\cF}(N, \delta)$ between the empirical value of any $f\in\cF$ and its true expected value in terms of the number of training samples $N$ and the confidence parameter $\delta$. $N_{\cF}(\varepsilon,\delta)$ is the number of samples required to achieve a prescribed error bound $\varepsilon$, while $\varepsilon_{\cF}(N, \delta)$ provides an error bound for any number $N$ of samples at hand. We provide bounds on $N_{\cF}(\varepsilon,\delta)$ and $\varepsilon_{\cF}(N,\delta)$ in terms of a common learning-theoretic measure of intrinsic complexity of $\cF$ called \emph{pseudo-dimension}, which is detailed in Section~\ref{sec:tree}.

In the context of Algorithm~\ref{alg:TS}, we study families of \emph{tree-constant} cost functions. A cost function $\cost:\cQ\to\R$ is tree constant if $\cost(Q)$ only depends on the tree built by Algorithm~\ref{alg:TS} on input $Q$ (an example is tree size). Let $\cost_{\mu,\lambda}(Q)$ denote this cost when Algorithm~\ref{alg:TS} is run using the scores $\ascore_{\mu} = \mu\cdot\ascore_1 + (1-\mu)\cdot\ascore_2$ and $\nscore_{\lambda} = \lambda\cdot\nscore_1 + (1-\lambda)\cdot\nscore_2$. We study the sample complexity of $\cF = \{\cost_{\mu,\lambda} : \mu, \lambda \in [0,1]\}$. We emphasize that we primarily interpret tree-constant functions as proxies for run-time. In the context of integer programming, tree size is one such measure. A strength of these guarantees is that they apply no matter how the parameters are tuned: optimally or suboptimally, manually or automatically. For \emph{any} configuration, these guarantees bound the difference between average performance over the training set and expected future performance on unseen IPs.

\section{Generalization guarantees for tree search}\label{sec:tree}

In order to derive our sample complexity guarantees, we first prove a key structural property: the behavior of Algorithm~\ref{alg:TS} is piecewise constant as a function of the node-selection score parameter $\lambda$ and the action-selection score parameter $\mu$. We give a high-level outline of our approach. We first assume that the conditional checks $\fathom(\tree, Q, \cdot) = \texttt{true}$ (lines~\ref{step:depth_fathom} and~\ref{step:action_fathom}) are suppressed. Let $\cA'$ denote Algorithm~\ref{alg:TS} without these checks (so $\cA'$ fathoms a node if and only if the depth limit is reached or if the node has no children). The behavior of $\cA'$ as a function of $\mu$ and $\lambda$ can be shown to be piecewise constant using the same argument as in Claim 3.4 of Balcan et al.~\cite{Balcan18:Learning}. Given this, our first main technical contribution (Lemma~\ref{lemma:rooted}) is a generalization of Claim 3.5 of Balcan et al.~\cite{Balcan18:Learning} that relates the behavior of $\cA'$ to Algorithm~\ref{alg:TS}. The argument in Balcan et al.~\cite{Balcan18:Learning} is specific to branching, but
we are able to prove our result in a much more general setting. Our second main technical contribution (Lemma~\ref{lemma:main}) is to establish piecewise structure when the node-selection score is controlled by $\lambda\in [0,1]$. The main reason for this auxiliary step of analyzing $\cA'$ is due to the fact that $\fathom$ is \emph{not} necessarily a path-wise function, and can depend on the state of the entire tree.

\begin{lemma}\label{lemma:rooted}
Fix $\mu\in [0,1]$. Let $\tree$ and $\tree'$ be the trees built by Algorithm~\ref{alg:TS} and $\cA'$, respectively, using the action-selection score $\mu\cdot\ascore_1 + (1-\mu)\cdot\ascore_2$. Let $Q$ be any node in $\tree$, and let $\tree_Q$ be the path from the root of $\tree$ to $Q$. Then, $\tree_Q$ is a rooted subtree of $\tree'$, no matter what node selection policy is used.
\end{lemma}

\begin{proof}
Let $t$ denote the length of the path $\tree_Q$. Let $\tree_Q$ be comprised of the sequence of nodes $(Q_1, \dots, Q_t)$ such that $Q_1$ is the root of $\tree$, $Q_t = Q$, and for each $\tau$, $Q_{\tau+1}\in\children(\tree_{Q_{\tau}},Q_{\tau}, A_{\tau})$ where $A_{\tau}\in\actions(\tree_{Q_{\tau}}, Q_{\tau})$ is the action selected by Algorithm~\ref{alg:TS} at node $Q_{\tau}$. We show that $(Q_1, \ldots, Q_t)$ is a rooted path in $\cT'$ as well.

Suppose for the sake of contradiction that this is not the case. Let $\tau\in\{2,\ldots, t\}$ be the minimal index such that $(Q_1,\ldots, Q_{\tau-1})$ is a rooted path in $\tree'$, but there is no edge in $\tree'$ from $Q_{\tau-1}$ to node $Q_{\tau}$. There are two possible cases:

{\em Case 1.} $Q_{\tau-1}$ was fathomed by $\cA'$.
This case is trivially not possible since whenever $\cA'$ fathoms a node, so does Algorithm~\ref{alg:TS} (recall $\cA'$ was defined by suppressing fathoming conditions of Algorithm~\ref{alg:TS}). 

{\em Case 2.} $Q_{\tau}\notin\children(\tree', Q_{\tau-1}, A'_{\tau-1})$ where $A'_{\tau-1}$ is the action taken by $\cA'$ at node $Q_{\tau-1}$. In this case, if $\children(\tree', Q_{\tau-1}, A'_{\tau-1})=\emptyset$, then $Q_{\tau-1}$ would be fathomed by $\cA'$, which cannot happen by the first case. Otherwise, if $\children(\tree', Q_{\tau-1}, A'_{\tau-1})\neq\emptyset$, we show that we arrive at a contradiction due to the fact that the scoring rules, action-set functions, and children functions are all path-wise. Let $A_{\tau-1}'$ denote the action taken by $\cA'$ at $Q_{\tau-1}$, and let $A_{\tau-1}$ denote the action taken by Algorithm~\ref{alg:TS} at $Q_{\tau-1}$.
Since $\actions$ is path-wise, $$\actions(\tree, Q_{\tau-1}) = \actions(\tree_{Q_{\tau-1}}, Q_{\tau-1}) = \actions(\tree', Q_{\tau-1}).$$ Since $\ascore_1$ and $\ascore_2$ are path-wise, we have \begin{align*}\mu \cdot\ascore_1(\tree, & Q_{\tau-1}, A) + (1-\mu)\cdot\ascore_2(\tree, Q_{\tau-1}, A) \\ &= \mu\cdot\ascore_1(\tree_{Q_{\tau-1}}, Q_{\tau-1}, A) + (1-\mu)\cdot\ascore_2(\tree_{Q_{\tau-1}}, Q_{\tau-1}, A)\\
&=\mu\cdot\ascore_1(\tree', Q_{\tau-1}, A) + (1-\mu)\cdot\ascore_2(\tree', Q_{\tau-1}, A).\end{align*} for all actions $A\in\actions(\tree_{Q_{\tau-1}}, Q_{\tau-1})$. Therefore Algorithm~\ref{alg:TS} and $\cA'$ choose the same action at node $Q_{t-1}$, that is, $A_{\tau-1} = A_{\tau-1}'$. Finally, since $\children$ is path-wise, we have $$\children(\tree, Q_{\tau-1}, A_{\tau-1}) = \children(\tree_{Q_{\tau-1}}, Q_{\tau-1}, A_{\tau-1}) = \children(\tree', Q_{\tau-1}, A_{\tau-1}).$$ Since $Q_{\tau}\in\children(\tree, Q_{\tau-1}, A_{\tau-1})$, this is a contradiction, which completes the proof.
\end{proof}

We use the following generalization of Claim 3.4 of Balcan et al.~\cite{Balcan18:Learning} that shows the behavior of $\cA'$ is piecewise constant. While their argument only applies to single-variable branching, our key insight is that the same reasoning can be readily adapted to handle any actions (including general branching constraints and cutting planes). The structure of our proof (which we defer to the appendix) is identical, but is modified to work in our more general setting. This style of analysis is similar in spirit to~\cite{Megiddo79:Combinatorial}.

\begin{lemma}\label{lemma:a'}
Let $\ascore_1$ and $\ascore_2$ be two path-wise action-selection scores. Fix the input root node $Q$. There are $T \leq k^{\Delta(\Delta-1)/2}b^\Delta$ subintervals $I_1, \dots, I_T$ partitioning $[0,1]$ where for any subinterval $I_{t}$, the action-selection score $\mu\cdot\ascore_1 + (1-\mu)\cdot\ascore_2$ results in the same tree built by $\cA'$ for all $\mu\in I_{t}$, no matter what node selection policy is used.
\end{lemma}

We now prove our main structural result for Algorithm~\ref{alg:TS}.

\begin{lemma}\label{lemma:main}
Let $\ascore_1$ and $\ascore_2$ be path-wise action-selection scores and let $\nscore_1$ and $\nscore_2$ be path-wise node-selection scores. Fix the input root node $Q$. There are $T\le k^{\Delta(9+\Delta)}b^{\Delta}$ rectangles partitioning $[0,1]^2$ such that for any rectangle $R_t$, the node-selection score $\lambda\cdot\nscore_1 + (1-\lambda)\cdot\nscore_2$ and the action-selection score $\mu\cdot\ascore_1 + (1-\mu)\cdot\ascore_2$ result in the same tree built by Algorithm~\ref{alg:TS} for all $(\mu, \lambda)\in R_t$. 
\end{lemma}

\begin{proof}
By Lemma~\ref{lemma:a'}, there is a partition of $[0,1]$ into subintervals $I_1\cup\cdots\cup I_T$ such that for all $\mu$ within a given subinterval, the tree built by $\cA'$ is invariant (independent of the node-selection score). Fix a subinterval $I_t$ of this partition. Let $\tree$ denote the tree built by Algorithm~\ref{alg:TS}. For each node $Q\in\tree$, let $\tree_Q$ denote the path from the root to $Q$ in $\tree$. Since $\nscore_1$ is path-wise, for any tree $\tree'$ containing $\tree_Q$ as a rooted path, $\nscore_1(\tree', Q) = \nscore_1(\tree_Q, Q)$. The same holds for $\nscore_2$. For every pair of nodes $Q_1,Q_2\in\tree$, let $\lambda(Q_1, Q_2)\in [0,1]$ denote the unique solution to \begin{align*}\lambda \cdot\nscore_1(\tree_{Q_1}, Q_1) &+ (1-\lambda)\cdot\nscore_2(\tree_{Q_1}, Q_1) \\ &= \lambda\cdot\nscore_1(\tree_{Q_2}, Q_2) + (1-\lambda)\cdot\nscore_2(\tree_{Q_2}, Q_2),\end{align*} if it exists (if there are either (1) no solutions or (2) infinitely many solutions, set $\lambda(Q_1, Q_2) = 0$). The thresholds $\lambda(Q_1, Q_2)$ for every pair of nodes $Q_1, Q_2\in\tree$ partition $[0,1]$ into subintervals such that for all $\lambda$ within a given subinterval, the total order over the nodes of $\tree$ induced by $\nscore_{\lambda}$ is invariant. In particular, this means that the node selected by each iteration of Algorithm~\ref{alg:TS} is invariant. Let $J_1\cup\cdots\cup J_S$ denote these subintervals induced by the thresholds over all subinterval $I_t\in\{I_1,\ldots, I_T\}$ established in Lemma~\ref{lemma:a'}.

We now show that this implies that the tree built by Algorithm~\ref{alg:TS} is invariant over all $(\mu,\lambda)$ within a given rectangle $I_t\times J_s$. Fix some rectangle $I_t\times J_s$. We proceed by induction on the iterations (of the while loop) of Algorithm~\ref{alg:TS}. For the base case (iteration $0$, before entering the while loop), the tree consists of only the root, so the hypothesis trivially holds. Now, suppose the statement holds up until the $j$th iteration, for some $j$. We analyze each line of Algorithm~\ref{alg:TS} to show that the behavior of the $j+1$st iteration is independent of $(\mu,\lambda)\in I_t\times J_s$. First, since $J_s$ determines the node selected at each iteration (as argued above), the node selected on the $j+1$st iteration (line~\ref{step:nsp}) is fixed, independent of $(\mu,\lambda)\in I_t\times J_s$. Denote this node by $Q$. Thus, whether $\texttt{depth}(Q)=\Delta$ is independent of $(\mu,\lambda)\in I_t\times J_s$, and similarly whether $\fathom(\tree, Q, \texttt{None}) = \texttt{true}$ is independent of $(\mu,\lambda)\in I_t\times J_s$ (line~\ref{step:depth_fathom}). This implies that whether or not $Q$ is fathomed at this stage is independent of $(\mu,\lambda)\in I_t\times J_s$. If $Q$ was fathomed, we are done. Otherwise, we argue that the action selected at line~\ref{step:action} is invariant over $(\mu,\lambda)\in I_t\times J_s$. By Lemma~\ref{lemma:a'}, $\cA'$ builds the same tree for all $\mu\in I_t$. Let $\tree_Q$ denote the path from the root to $Q$ in this tree. By Lemma~\ref{lemma:rooted}, $\tree_Q$ is the path from the root to $Q$ in the tree built by Algorithm~\ref{alg:TS} as well. The action selected at $Q$ by $\cA'$ is invariant over $\mu\in I_t$ (by Lemma~\ref{lemma:a'}). Therefore, since $\actions$, $\ascore_1$, and $\ascore_2$ are path-wise, the action $A$ selected by Algorithm~\ref{alg:TS} at $Q$ is invariant over $\mu\in I_t$. Finally, $\fathom(\tree, Q, A)$ and $\children(\tree, Q, A)$ are completely determined, so the execution of the remaining conditional statement (line~\ref{step:action_fathom} to line~\ref{step:add_children}) is invariant over $(\mu, \lambda)\in I_t\times J_s$. Thus, the entire iteration of Algorithm~\ref{alg:TS} is invariant over $(\mu, \lambda)\in I_t\times J_s$, which completes the induction.

Finally, we count the total number of rectangles in our partition of $[0,1]^2$. For each interval $I_t$ in the partition established in Lemma~\ref{lemma:a'}, we obtained a partition of $I_t\times [0,1]$ into rectangles induced by at most $\binom{|\tree|}{2}$ thresholds, which consists of at most at most $$1+\binom{(k^{\Delta+1}-1)/(k-1)}{2}\le 1 + \left(\frac{k^{\Delta+1}-1}{k-1}\right)^2\le k^{5\Delta}$$ subintervals. Accounting for every interval $I_t\in\{I_1,\ldots, I_T\}$ in the partition from Lemma~\ref{lemma:a'}, we get a total of $Tk^{5\Delta}\le k^{\Delta(9+\Delta)/2}b^{\Delta}$ rectangles, as desired.
\end{proof}

We now derive generalization guarantees for the collection $\cF = \{\cost_{\mu,\lambda}:(\mu,\lambda)\in[0,1]^2\}$ where $\cost$ is any tree-constant function, such as tree size. We do this by bounding the \emph{pseudo-dimension} of $\cF$, which is a combinatorial measure of intrinsic complexity of a class of real valued functions. The pseudo-dimension of $\cF$, denoted by $\pdim(\cF)$, is the largest positive integer $N$ such that there exist $N$ nodes $Q_1,\ldots, Q_N\in\cQ$ and $N$ thresholds $r_1,\ldots, r_N\in\R$ such that $|\{(\sign(f(Q_1)-r_1),\ldots, \sign(f(Q_N)-r_N)):f\in\cF\}|=2^N.$ A well-known result in learning theory~\cite{Anthony09:Neural} states that if functions in $\cF$ have bounded range $[-H, H]$, then $$N_{\cF}(\varepsilon,\delta) = O\left(\frac{H^2}{\varepsilon^2}\left(\pdim(\cF) + \ln(1/\delta)\right)\right)\text{ and }\varepsilon_{\cF}(N,\delta) = O\left(H\sqrt{\frac{\pdim(\cF)+\ln(1/\delta)}{N}}\right).$$ When each function in $\cF$ maps to $\{0,1\}$, the pseudo-dimension is more commonly referred to as the \emph{VC dimension}.

Bounding the pseudo-dimension is a simple instantiation of the general framework provided by Balcan et al.~\cite{Balcan21:How} with the piecewise structure established in Lemma~\ref{lemma:main}. Balcan et al.'s~\cite{Balcan21:How} main result gives pseudo-dimension bounds for families of piecewise structured functions in terms of the VC dimension of the class of $0/1$ classifiers defining the boundaries of the functions, the number of classifiers defining the boundaries, and the pseudo-dimension of the family of functions when restricted to each piece. (Strictly, this result is in terms of the dual classes of the boundary and piece functions. However, since the dual class of all linear separators is the set of all linear separators, we omit this detail for simplicity.)

\begin{theorem}\label{theorem:main}
Let $\cost(Q)$ be any tree-constant cost function, and let $\cost_{\mu,\lambda}(Q)$ be the cost of the tree built by Algorithm~\ref{alg:TS} on input root node $Q$ using action-selection score parameterized by $\mu$ and node-selection score parameterized by $\lambda$. Then, $\pdim(\{\cost_{\mu,\lambda}\}) = O(\Delta^2\log k + \Delta\log b)$.
\end{theorem}

\begin{proof}
By Lemma~\ref{lemma:main}, there are at most $T = k^{\Delta(9+\Delta)}b^{\Delta}$ rectangles partitioning $[0,1]^2$ such that for a fixed input node $Q$, $\cost_{\mu,\lambda}(Q)$ is constant over each rectangle as a function of $\mu,\lambda$. These $T$ rectangles can be defined by $T$ thresholds on $[0,1]$ corresponding to $\mu$ and $T$ thresholds on $[0,1]$ corresponding to $\lambda$. Thus, the $T$ rectangles can be identified by $T^2 = k^{2\Delta(9+\Delta)}b^{2\Delta}$ linear separators in $\R^2$. The VC dimension of linear separators in $\R^2$ is $O(1)$. The pseudo-dimension of the set of constant functions is also $O(1)$. Plugging these quantities into the main theorem of Balcan et al.~\cite{Balcan21:How} yields the theorem statement. \end{proof}


\subsection{Multiple actions}\label{sec:multiple_actions}

Theorem~\ref{theorem:main} can be easily generalized to the case where there are multiple actions of different types taken at each node of Algorithm~\ref{alg:TS}. Specifically, there are now $d$ path-wise action-set functions $\actions_1,\ldots,\actions_d$, and at line~\ref{step:action} of Algorithm~\ref{alg:TS} we take one action of each type, that is, we select action $A_1\in\actions_1(\tree, Q)$, $A_2\in\actions_2(\tree, Q)$, and so on. The functions $\fathom$ and $\children$ then depend on all $d$ actions taken at node $Q$. We assume that there are two scoring rules $\ascore^i_1$ and $\ascore^i_2$ for each action type $i = 1,\ldots, d$. Algorithm~\ref{alg:TS} can then be parameterized by $(\vec{\mu}, \lambda)$, where $\vec{\mu}\in\R^d$ is a vector of parameters controlling each action, so the $i$th action is selected to maximize $\mu_i\cdot\ascore^i_1 + (1-\mu_i)\cdot\ascore^i_2$. Then, as long as $d = O(1)$, we get the same pseudo-dimension bound. We assume $b$ is a uniform upper bound on the size of $\actions_i$ for any $i$. The proof is nearly identical, and we defer it to the appendix (which also contains more details on the multiple-action setup).

\begin{theorem}\label{theorem:multiple_actions}
Let $\cost(Q)$ be any tree-constant cost function, and let $\cost_{\vec{\mu},\lambda}(Q)$ be the cost of the tree built by Algorithm~\ref{alg:TS} on input root node $Q$ using action-selection scores parameterized by $\vec{\mu}\in\R^d$, where $d = O(1)$, and node-selection score parameterized by $\lambda$. Then, $\pdim(\{\cost_{\vec{\mu},\lambda}\}) = O(\Delta^2\log k + \Delta\log b)$.
\end{theorem}

\section{Branch-and-cut for integer programming}\label{sec:branch_and_cut}

We now instantiate our main results with the three main components of the B\&C algorithm: branching, cutting planes, and node selection, used to solve IPs $\max\{\vec{c}^T\vec{x}:A\vec{x}\le\vec{b}, \vec{x}\ge 0, \vec{x}\in\Z^n\}$ where $\vec{c}\in\R^n$, $A\in\Z^{m\times n}$, $\vec{b}\in\Z^m$. The function $\fathom(\tree, Q, A)$ outputs $\texttt{true}$ if after having taken action $A$ the LP relaxation at $Q$ is integral, infeasible, or worse than the best integral solution found so far in $\tree$. The function $\children(\tree, Q, A)$ outputs the two subproblems generated by the branching procedure on the IP at $Q$ after having taken action $A$. For simplicity we refer only to IPs, but everything in our discussion applies to mixed IPs as well. In our model of tree search, node selection is controlled by $\lambda$. Cutting planes and branching are types of actions and controlled by $\mu$.

\subsection{Branching}

In this section, we provide guarantees for branching.
Throughout this section we assume $\Delta = O(n)$, as is the case with single-variable branching.

\subsubsection{Multivariable branching constraints}
It is well known that allowing for more general generation of branching constraints can result in smaller B\&C trees. Gilpin and Sandholm~\cite{Gilpin11:Information} studied multivariable branches of the form $\textstyle\sum_{i\in S}\vec{x}[i]\le\left\lfloor\sum_{i\in S}\vec{x}^*_{\LP}[i]\right\rfloor$, $\sum_{i\in S}\vec{x}[i]\ge\left\lceil\sum_{i\in S}\vec{x}^*_{\LP}[i]\right\rceil$ where $S$ is a subset of the integer variables such that $\sum_{i\in S}\vec{x}^*_{\LP}[i]\notin\Z$. Here, $\actions(\tree, Q) = 2^{[n]}$, so, $\pdim(\{\cost_{\mu,\lambda}\}) = O(n^2)$. So our sample complexity bound for multivariable branching constraints is, surprisingly, only a constant factor worse than the bound for single-variable branching constraints.

We give a simple example where B\&C using only single variable branches builds a tree of exponential size, while a single branch on the entire set of variables at the root yields two infeasible subproblems (and a B\&C tree of size $3$).

\begin{theorem}
For any $n$, there is an IP with two constraints and $n$ variables such that with only single variable branches, B\&C builds a tree of size $2^{(n-1)/2}$, while with a suitable multivariable branch, B\&C builds a tree of size three.
\end{theorem}

\begin{proof}
Let $n$ be an odd positive integer. Consider the infeasible IP $\max\{\sum_{i=1}^nx[i] : 2\sum_{i=1}^nx[i]= n, \vec{x}\in\{0,1\}^n\}$. Jeroslow~\cite{Jeroslow74:Trivial} proved that with only single-variable branches, B\&C builds a tree with $2^{(n-1)/2}$ nodes to determine infeasibility. However, with a suitable multivariable branch, B\&C will build a tree of constant size. The optimal solution to the LP relaxation of the IP is attained when all variables are set to $1/2$. A multivariable branch on all $n$ variables produces the two subproblems with constraints $\sum_{i=1}^nx[i]\le\lfloor n/2\rfloor$ and $\sum_{i=1}^nx[i]\ge\lceil n/2\rceil$, respectively. Since $n$ is odd, $\lfloor n/2 \rfloor < n/2$ and $\lceil n/2\rceil > n/2$, so the LP relaxations of both subproblems are infeasible. Thus, B\&C builds a tree with three nodes. 
\end{proof}

Yang et al.~\cite{Yang21:Multivariable} provide more examples of situations where multivariable branching yields dramatic improvements in tree size over single variable branching. They also perform a computational evaluation of a few different strategies for generating multivariable branching constraints. Yang et al.~\cite{Yang20:Learning} explore gradient-boosting for learning to mimic strong branching for multiple variables. 

\subsubsection{Branching on general disjunctions}

Branching constraints can be even more general than multivariable branches. Given any integer vector $\vec{\pi}\in\Z^n$ and any integer $\pi_0\in\Z$ (jointly referred to as a \emph{disjunction}), the constraints $\vec{\pi}^T\vec{x}\le\pi_0$ or $\vec{\pi}^T\vec{x}\ge\pi_0 + 1$ represent a valid partition of the feasible region into subproblems. Owen and Mehrotra~\cite{Owen01:Experimental} ran the first experiments demonstrating that branching on general disjunctions can lead to significantly smaller tree sizes. Subsequent works have posed different heuristics to select disjunctions to branch on~\cite{Fischetti02:Local,Mahajan09:Experiments}.

In practice it is known that additional IP constraints should not have coefficients that are too large. If $C$ is a bound on the magnitude of the coefficient of any disjunction, then $\actions(\tree, Q) = \{-C,\ldots,C\}^{n+1}$, so $\pdim(\{\cost_{\mu,\lambda}\}) = O(n^2\log C)$. Karamanov and Cornu\'{e}jols~\cite{Karamanov11:Branching} conduct a computational evaluation of disjunctions derived from Gomory mixed-integer cuts. In this setting, $\actions(\tree, Q)$ is the set of $m$ or fewer disjunctions corresponding to the $m$ or fewer Gomory mixed-integer cuts derived from the simplex tableau from solving the LP relaxation of $Q$. In this case, $\pdim(\{\cost_{\mu,\lambda}\}) = O(n^2 + n\log m)$.

\subsection{Cutting planes}

The action set can also correspond to cutting planes used to refine the feasible region of the IP at any stage of B\&C. Here, $\actions(\tree, Q)$ is any set of cutting planes derived solely using the path from the root to the IP at $Q$. Examples include the set of Chv\'{a}tal-Gomory (CG) derived from the simplex tableau~\cite{Gomory58:Outline}, and various combinatorial families of cutting planes such as clique cuts, odd-hole cuts, and cover cuts. The set $\actions(\tree, Q)$ can also consist of sequences of cutting planes, representing adding several cutting planes to the IP in waves. For example, the set of all sequences of $w$ CG cuts generated from the simplex tableau for an IP with $m$ constraints has size at most $m^w$ (regardless of whether the LP is resolved after each cut). The number of such cutting planes provided by the LP tableau at any node in the tree is at most $O(m+nw)$ (the original IP has $m$ constraints, and after at most $n$ branches there are an additional $n$ branching constraints and at most $nw$ cutting planes), which means that $|\actions(\tree, Q)|\le O(m+nw)^w$. Thus, $\pdim(\{\cost_{\mu,\lambda}\}) = O(n^2 + nw\log(m+nw))$.

We can also handle arbitrary CG cuts (not just ones from the LP tableau). Balcan et al.~\cite{Balcan21:Sample} proved that given an IP with feasible region $\{\vec{x}\in\Z^n : A\vec{x}\le\vec{b}, \vec{x}\ge 0\}$, even though there are infinitely many CG cut parameters, there are effectively only $O(w2^w\norm{A}_{1, 1} + 2^w\norm{\vec{b}}_1+ nw)^{1+mw}$ distinct sequences of cutting planes that $w$ CG cut parameters can produce. At any node in the B\&C tree, the number of constraints is at most $O(m+nw)$. So, on the domain of IPs with $\norm{A}_{1,1}\le\alpha$ and $\norm{\vec{b}}_1\le\beta$, $|\actions(\tree, Q)|\le O(w2^w\alpha + 2^w\beta+nw)^{1+w\cdot O(m+nw)}$. Thus, $\pdim(\{\cost_{\mu,\lambda}\}) = O(n^2w^3m\log(\alpha+\beta+n))$.

\subsubsection{Experiments on cover cuts for the multiple knapsack problem}

In this section, we demonstrate via experiments that tuning a convex combination of scoring rules to select cuts can lead to dramatically smaller branch-and-cut trees when done in a data-dependent manner. We study the classical NP-hard \emph{multiple knapsack problem}: given a set $N$ of items where each item $i\in N$ has a value $p_i\ge 0$ and a weight $w_i\ge 0$, and a set $K$ of knapsacks where each knapsack $k\in K$ has a capacity $W_k\ge 0$, the goal is to find a feasible packing of the items into the knapsacks of maximum value. We assume, without loss of generality, that the items are labeled in descending order of weight, that is, $w_1\ge w_2\ge\cdots\ge w_{|N|}$. This problem can be formulated as the following binary IP: $$\begin{array}{ll} \text{maximize} & \sum_{i\in N}\sum_{k\in K} p_ix_{k,i}\\
\text{subject to} &  \sum_{i\in N}w_ix_{k, i} \le W_k\hfill\forall\,k\in K\\
&  \sum_{k\in K}x_{k,i}\le 1\hfill\forall\,i\in N\\
& x_{k, i}\in\{0,1\}\hfill\qquad\forall\, i\in N, k\in K
\end{array}$$ A subset $C\subseteq N$ of items is called a \emph{cover} for knapsack $k\in K$ if $\sum_{i\in C}w_i > W_k$. If $C$ is a cover, no feasible solution can have $x_{k, i} = 1$ for all $i\in C$, so $\sum_{i\in C} x_{k, i}\le |C| - 1$ is a valid constraint---called a \emph{cover cut}. When $C$ is minimal (that is, $C\setminus\{i\}$ is not a cover for every $i\in C$), such cover cuts help tighten the knapsack IP by cutting off fractional LP solutions. We generate (a subset of all) cover cuts for each knapsack $k$ as follows: for each $i\in N$, let $j > i$ be minimal such that $C = \{i,i+1,\ldots, j\}$ is a cover for $k$ (if such a $j$ exists). Since $w_i\ge w_j$ for $j > i$, $C$ is a minimal cover, and moreover the \emph{extended cover cut} $\sum_{i=1}^j x_i\le |C| - 1$ is valid and dominates the minimal cover cut $\sum_{i\in C}x_i\le |C| - 1$. Extended cover cuts generated from minimal covers are known to be facet defining for the integer hull under certain natural conditions~\cite{Conforti14:Integer}.

We investigate the relationship between three scoring rules for cutting planes. The first is \emph{efficacy} ($\texttt{E}$), which is the perpendicular distance from the current LP solution to the cutting plane. The second is \emph{parallelism} ($\texttt{P}$), which measures the angle between the objective and the normal vector to the cutting plane. The third is \emph{directed cutoff} ($\texttt{D}$), which is the distance from the current LP solution to the cutting plane along the direction of the line segment connecting the LP solution to the current best incumbent integer solution. More details, including explicit formulas, can be found in~\cite{Balcan21:Sample} and references therein.

\begin{figure}[t]
     \centering
     \begin{subfigure}[b]{0.32\textwidth}
         \centering
         \includegraphics[width=\textwidth]{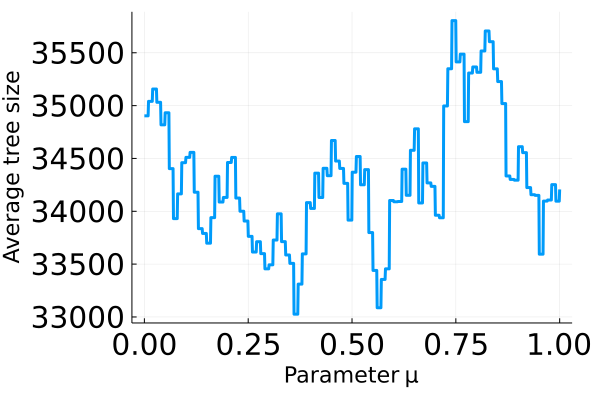}
         \caption{$\mu\cdot\texttt{E} + (1-\mu)\cdot\texttt{P}$}
         \label{fig:chvatal_35i_2k_ep}
     \end{subfigure}
     \hfill
     \begin{subfigure}[b]{0.32\textwidth}
         \centering
         \includegraphics[width=\textwidth]{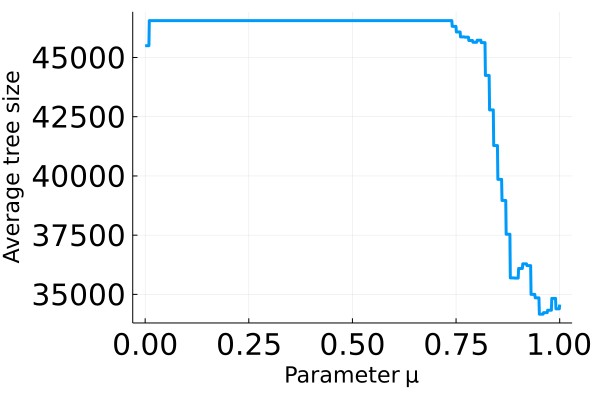}
         \caption{$\mu\cdot\texttt{E} + (1-\mu)\cdot\texttt{D}$}
         \label{fig:chvatal_35i_2k_ed}
     \end{subfigure}
     \hfill
     \begin{subfigure}[b]{0.32\textwidth}
         \centering
         \includegraphics[width=\textwidth]{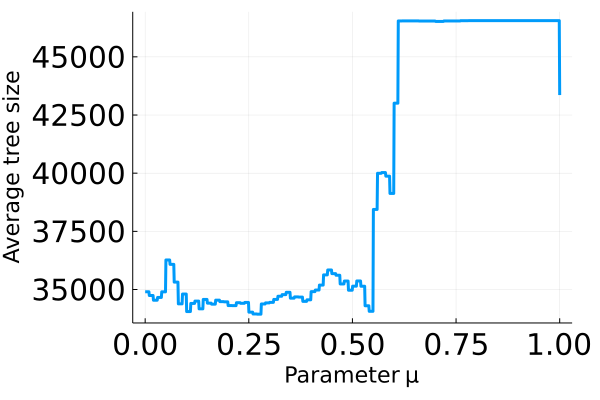}
         \caption{$\mu\cdot\texttt{D} + (1-\mu)\cdot\texttt{P}$}
         \label{fig:chvatal_35i_2k_dp}
     \end{subfigure}
        \caption{Chv\'{a}tal distribution with $35$ items and $2$ knapsacks.}
        \label{fig:chvatal_35i_2k}
\end{figure}

\begin{figure}[t]
     \centering
     \begin{subfigure}[b]{0.32\textwidth}
         \centering
         \includegraphics[width=\textwidth]{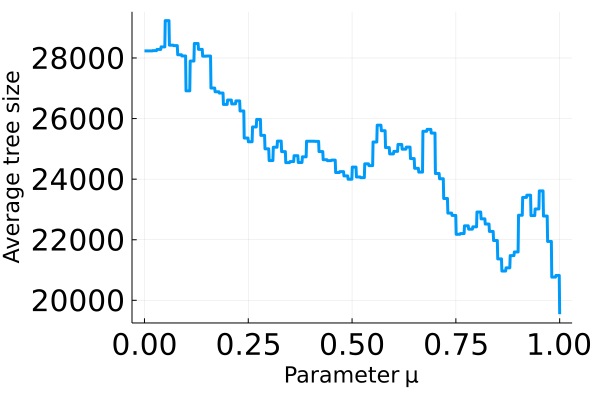}
         \caption{$\mu\cdot\texttt{E} + (1-\mu)\cdot\texttt{P}$}
         \label{fig:chvatal_35i_3k_ep}
     \end{subfigure}
     \hfill
     \begin{subfigure}[b]{0.32\textwidth}
         \centering
         \includegraphics[width=\textwidth]{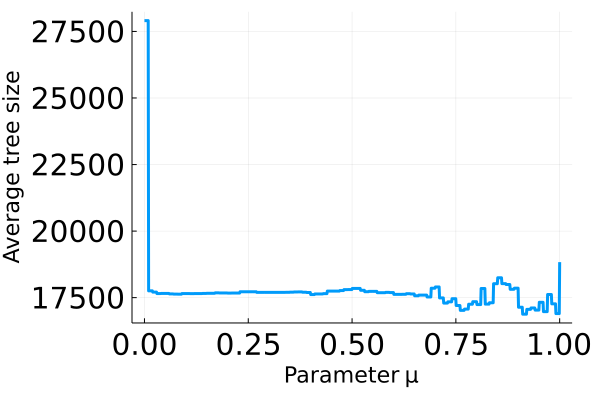}
         \caption{$\mu\cdot\texttt{E} + (1-\mu)\cdot\texttt{D}$}
         \label{fig:chvatal_35i_3k_ed}
     \end{subfigure}
     \hfill
     \begin{subfigure}[b]{0.32\textwidth}
         \centering
         \includegraphics[width=\textwidth]{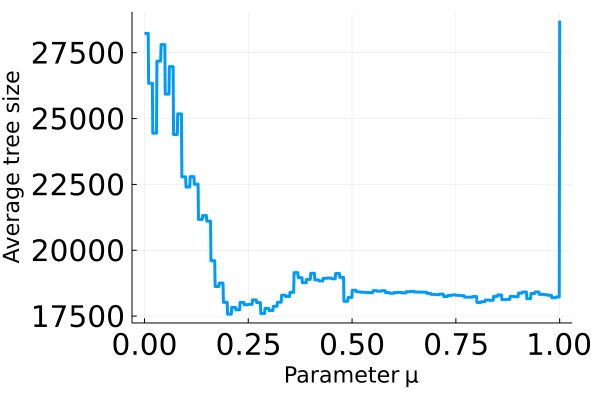}
         \caption{$\mu\cdot\texttt{D} + (1-\mu)\cdot\texttt{P}$}
         \label{fig:chvatal_35i_3k_dp}
     \end{subfigure}
        \caption{Chv\'{a}tal distribution with $35$ items and $3$ knapsacks.}
        \label{fig:chvatal_35i_3k}
\end{figure}

We consider two specific instances of the multiple knapsack problem, which are loosely based on a class of knapsack problems introduced by Chv\'{a}tal that are difficult to solve with vanilla branch-and-bound~\cite{Chvatal80:Hard,Yang21:Multivariable}. In the first, $p_i = w_i$ for all $i\in N$, and $W_k = \lfloor(\sum_{i\in N}w_i)/2|K|\rfloor + (k-1)$ for each $k = 1,\ldots, |K|$. In the second, $p_i = w_{|N| - i + 1}$, so the most valuable item is the lightest and the least valuable item is the heaviest, and $W_k$ is defined as in the first type. We call the first class of problems \emph{Chv\'{a}tal instances} and the second class \emph{reverse Chv\'{a}tal instances}. For a given $N, K$, we generate (reverse) Chv\'{a}tal instances by drawing each weight independently as $w_i = \lfloor z_i\rfloor$, where $z_i\sim \mathcal{N}(50, 2)$, and sorting the items by weight in descending order. 

\begin{figure}[t]
     \centering
     \begin{subfigure}[b]{0.32\textwidth}
         \centering
         \includegraphics[width=\textwidth]{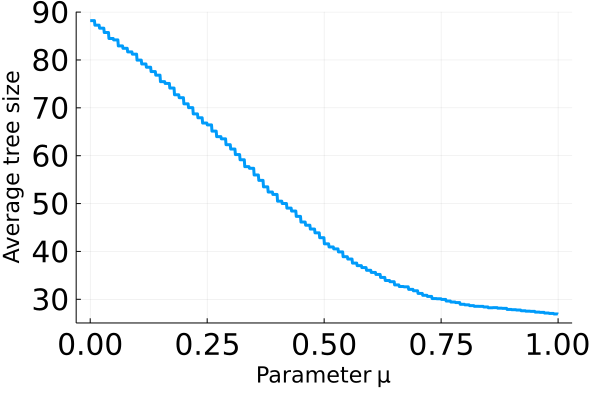}
         \caption{$\mu\cdot\texttt{E} + (1-\mu)\cdot\texttt{P}$}
         \label{fig:reverse_chvatal_100i_10k_ep}
     \end{subfigure}
     \hfill
     \begin{subfigure}[b]{0.32\textwidth}
         \centering
         \includegraphics[width=\textwidth]{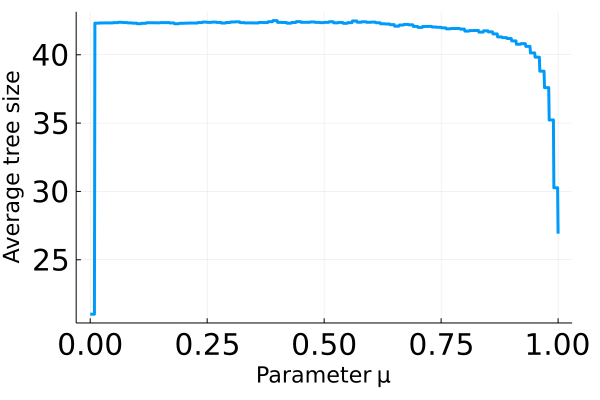}
         \caption{$\mu\cdot\texttt{E} + (1-\mu)\cdot\texttt{D}$}
         \label{fig:reverse_chvatal_100i_10k_ed}
     \end{subfigure}
     \hfill
     \begin{subfigure}[b]{0.32\textwidth}
         \centering
         \includegraphics[width=\textwidth]{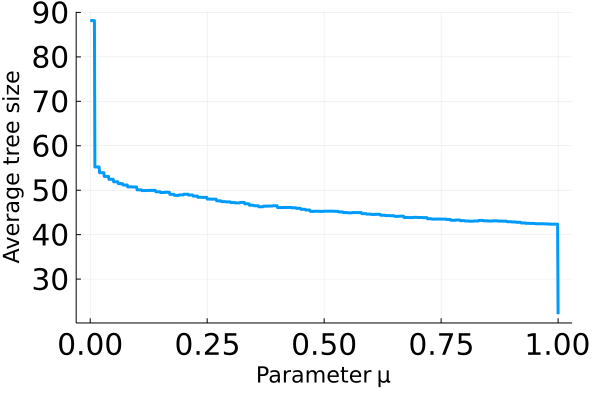}
         \caption{$\mu\cdot\texttt{D} + (1-\mu)\cdot\texttt{P}$}
         \label{fig:reverse_chvatal_100i_10k_dp}
     \end{subfigure}
        \caption{Reverse Chv\'{a}tal distribution with $100$ items and $10$ knapsacks.}
        \label{fig:reverse_chvatal_100i_10k}
\end{figure}

\begin{figure}[t]
     \centering
     \begin{subfigure}[b]{0.32\textwidth}
         \centering
         \includegraphics[width=\textwidth]{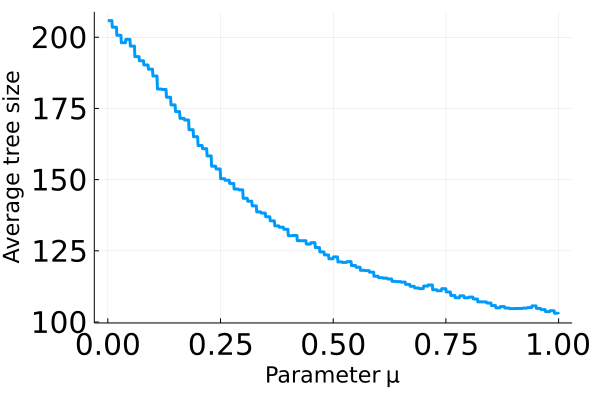}
         \caption{$\mu\cdot\texttt{E} + (1-\mu)\cdot\texttt{P}$}
         \label{fig:reverse_chvatal_100i_15k_ep}
     \end{subfigure}
     \hfill
     \begin{subfigure}[b]{0.32\textwidth}
         \centering
         \includegraphics[width=\textwidth]{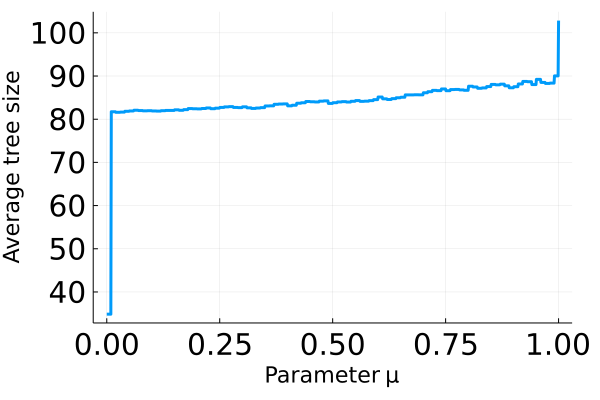}
         \caption{$\mu\cdot\texttt{E} + (1-\mu)\cdot\texttt{D}$}
         \label{fig:reverse_chvatal_100i_15k_ed}
     \end{subfigure}
     \hfill
     \begin{subfigure}[b]{0.32\textwidth}
         \centering
         \includegraphics[width=\textwidth]{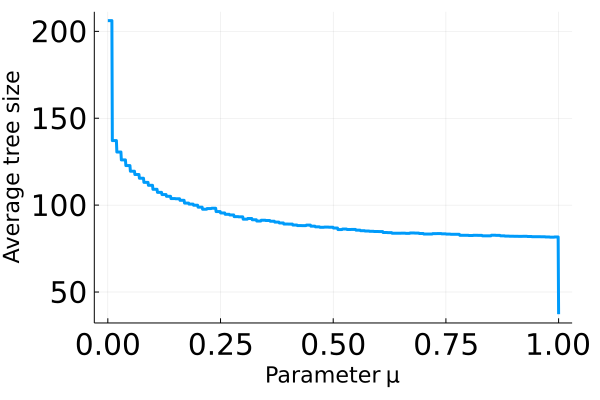}
         \caption{$\mu\cdot\texttt{D} + (1-\mu)\cdot\texttt{P}$}
         \label{fig:reverse_chvatal_100i_15k_dp}
     \end{subfigure}
        \caption{Reverse Chv\'{a}tal distribution with $100$ items and $15$ knapsacks.}
        \label{fig:reverse_chvatal_100i_15k}
\end{figure}

In our experiments, we add (whenever possible) two extended cover cuts obtained in the aforementioned manner at every node of the B\&C tree. The two cuts chosen are the two with the highest score $\mu\cdot\ascore_1 + (1-\mu)\cdot\ascore_2$ among all extended cover cuts that are violated by the current LP optimum, where $\ascore_1,\ascore_2\in\{\texttt{E},\texttt{D},\texttt{P}\}$. Figures~\ref{fig:chvatal_35i_2k}-\ref{fig:reverse_chvatal_100i_15k} display the average tree size over $1000$ samples for different Chv\'{a}tal and reverse Chv\'{a}tal distributions as a function of $\mu$, where the domain $[0,1]$ of $\mu$ is discretized in increments of $0.01$. We ran our experiments using the Python API of CPLEX 12.10 with default cut generation turned off. All other aspects of B\&C (e.g. variable and node selection) are controlled by the default settings of CPLEX. The key takeaway of our plots is that tuning a convex combination of scoring rules can lead to significant savings in B\&C tree size, and that this tuning must be done with the IP distribution in mind. No single parameter produces small trees for all the distributions we considered, and in fact a $\mu$ that minimizes tree size for one distribution can result in the largest trees for another (as in Figures~\ref{fig:chvatal_35i_3k_ed} and \ref{fig:reverse_chvatal_100i_15k_ed}, for example). Furthermore, many of the plots display discernible trends (and in some cases are quite smooth), suggesting that the number of samples required to avoid overfitting in practice can be significantly smaller than our theoretical bounds.

\subsection{Improved bounds for branch-and-cut}

To allow node selection, branching, and cutting-plane selection to be tuned simultaneously, we apply Theorem~\ref{theorem:multiple_actions}. This allows us to bound the pseudo-dimension of the family of functions $\{\cost_{\mu_1, \mu_2, \lambda}\}$, where $\mu_1$ controls branching, $\mu_2$ controls cutting-plane selection, and $\lambda$ controls node selection. Let $\actions_1(\tree, Q)$ denote the set of branching actions available at $Q$, and let $\actions_2(\tree, Q)$ denote the set of cutting planes available at $Q$. Let $b_1, b_2\in\N$ be such that $\actions_1(\tree, Q)\le b_1$ and $\actions_2(\tree, Q)\le b_2$ for all $\tree$ and all $Q\in \tree$. Fix two branching scores $\ascore^1_1, \ascore^1_2$, fix two cutting-plane selection scores $\ascore^2_1, \ascore^2_2$, and fix two node-selection scores $\nscore_1, \nscore_2$.

\begin{theorem}
Let $\cost(Q)$ be any tree-constant cost function, and let $\cost_{\mu_1,\mu_2, \lambda}$ be the cost of the tree built by B\&C using branching score $\mu_1\cdot\ascore^1_1 + (1-\mu_1)\cdot\ascore^1_2$, cutting-plane selection score $\mu_2\cdot\ascore^2_1 + (1-\mu_2)\cdot\ascore^2_2$, and node-selection score $\lambda\cdot\nscore_1 + (1-\lambda)\cdot\nscore_2$. Then, with $\Delta = O(n)$, $\pdim(\{\cost_{\mu_1,\mu_2,\lambda}\}) = O(n^2 + n\log(b_1+b_2))$.
\end{theorem}

\subsubsection{Comparison to existing bounds}\label{sec:comparison}

Balcan et al.~\cite{Balcan21:Sample} give a pseudo-dimension bound for tree search with a linear dependence on a cap $\kappa$ on the number of nodes allowed in any tree. Their pseudo-dimension bound in our setting is $\pdim(\{\cost_{\mu_1,\mu_2,\lambda}\})=O(\kappa\log\kappa + \kappa\log b_1 + \kappa\log b_2)$. While $\kappa$ is treated as a constant, it can be a prohibitively large quantity. In fact, without explicitly enforcing a limit on the number of nodes expanded by B\&C, Balcan et al.~\cite{Balcan21:Sample} obtain a pseudo-dimension bound of $O(2^n(\log b_1 + \log b_2))$. Balcan et al.~\cite{Balcan18:Learning} use the path-wise property to prove that $\pdim(\{\cost_{\mu}\}) = O(n^2)$ for single-variable branching, but for the case where branching is the only tunable component of B\&C (and node selection is fixed).

\section{Conclusions and future research}

We presented a general model of tree search and proved sample complexity guarantees for this model that improve and generalize upon the recent sample complexity theory for configuring branch-and-cut.
There are many interesting and open directions for future research.
One compelling open question is to obtain pseudo-dimension bounds when action sets are infinite. Balcan et al.~\cite{Balcan21:Sample} alluded to this question in the case of cutting planes, and neither the techniques of their work nor the techniques of the present work can handle, for example, important infinite cutting-plane families such as the class of Gomory mixed-integer cuts, or the infinitely many valid disjunctions that could be branched on. Beyond integer programming, our model of tree search could potentially be applied to completely different problem domains that exhibit tree structure.
Another direction is to extend our results to convex combinations of $\ell > 2$ scoring rules $\mu_1\score_1 + \dots \mu_{\ell}\score_{\ell},$ as Balcan et al.~\cite{Balcan18:Learning} do in the special case of single-variable branching. However, their pseudo-dimension bound grows exponentially in the number of variables $n$ in that special case; developing techniques that lead to a polynomial dependence on $n$ remains a challenging open question.

\subsection*{Acknowledgements}
This material is based on work supported by the National Science Foundation under grants CCF-1733556, CCF-1910321, IIS-1901403, and SES-1919453, the ARO under award W911NF2010081, the Defense Advanced Research Projects Agency under cooperative agreement HR00112020003, a Simons Investigator Award, an AWS Machine Learning Research Award, an Amazon Research Award, a Bloomberg Research Grant, and a Microsoft Research Faculty Fellowship.

\bibliographystyle{plainnat}
\bibliography{dairefs}

\appendix

\section{Analysis of $\cA'$}

\begin{proof}[Proof of Lemma~\ref{lemma:a'}]
Let $\tree$ denote the tree built by $\cA'$. For $i \in [\Delta]$, let $\tree[i]$ denote the restriction of $\tree$ to nodes of depth at most $i$. Let $\ascore_{\mu} = \mu\cdot\ascore_1 + (1-\mu)\cdot\ascore_2$. We prove the lemma by induction on $i$. In particular, we show that for each $i\in[\Delta]$, there are $k^{i(i-1)/2}b^{i}$ subintervals partitioning $[0,1]$ such that $\tree[i]$ is invariant over all $\mu$ within any given subinterval. Since $\tree[\Delta] = \tree$, this implies the lemma statement. The base case of $i=1$ is trivial since $\tree[1]$ consists of only the root.

Now, suppose the statement holds for some $i\in\{1,\ldots,\Delta-1\}$. That is, there are $T\le k^{i(i-1)/2}b^{i}$ disjoint intervals $I_1\cup\cdots\cup I_T = [0,1]$ such that $\tree[i]$ is invariant over all $\mu$ within any given subinterval (our inductive hypothesis). Fix one of these subintervals $I_t$. We subdivide $I_t$ into subintervals such that $\tree[i+1]$ is invariant within each one of these smaller subintervals. Let $Q$ be any leaf of $\tree[i]$, and for $\mu\in I_t$ let $\tree_{\mu}$ denote the state of the tree using $\ascore_{\mu}$ at the point that $Q$ is selected. Since $i < \Delta$, $Q$ is not fathomed at line~\ref{step:depth_fathom}, regardless of $\mu$. Next, since $\actions$ is path-wise, the actions available at $Q$ depend only on the path $\tree_Q$ from the root of $\tree$ to $Q$, which, by the inductive hypothesis, is invariant over all $\mu\in I_t$. That is, $\actions(\tree_{\mu}, Q) = \actions(\tree_Q, Q)$ for all $\mu\in I_t$. Then, $\ascore_{\mu}$ with parameter $\mu$ will select action $A\in\actions(\tree_Q, Q)$ if and only if \begin{align*}A &= \argmax_{A_0\in\actions(\tree_Q, \node)}\mu\cdot\ascore_1(\tree_{\mu}, Q, A_0) + (1-\mu)\cdot\ascore_2(\tree_{\mu}, Q, A_0) \\ &= \argmax_{A_0\in\actions(\tree_Q, \node)}\mu\cdot\ascore_1(\tree_Q, Q, A_0) + (1-\mu)\cdot\ascore_2(\tree_Q, Q, A_0),\end{align*} where the second equality follows from the fact that $\ascore_1$ and $\ascore_2$ are path-wise. Thus, for a fixed $A_0$, $\ascore_{\mu}$ is linear in $\mu$, so for each $A_0$ there is at most one subinterval of $[0,1]$ such that for all $\mu$ in that subinterval, $A_0$ maximizes $\ascore_{\mu}$. Thus, each leaf of $\tree[i]$ contributes at most $b$ subintervals such that for $\mu$ within a given subinterval, the action selected at each leaf of $\tree[i]$ is invariant. $\tree[i]$ consists of at most $k^i$ leaves, so this is a total of at most $k^ib$ subintervals. Now, since the action $A$ selected at each leaf $Q$ of $\tree[i]$ is invariant, the set of children $\children(\tree_{\mu}, Q, A) = \children(\tree_Q, Q, A)$ of $Q$ added to the tree is also invariant, using the fact that $\children$ is path-wise. This shows that within every subinterval, $\tree[i+1]$ is invariant. The total number of subintervals is, by the induction hypothesis, at most $k^{i(i-1)/2}b^{i}\cdot k^ib =  k^{(i+1)i/2}b^{i+1}$, as desired.
\end{proof}

\section{Multiple actions}

Let $\actions_1,\ldots,\actions_d$ be path-wise. The multi-action version of Algorithm~\ref{alg:TS} is given by Algorithm~\ref{alg:mTS}.
\begin{algorithm}[t]
	\caption{Tree search with multiple actions}\label{alg:mTS}
	\begin{algorithmic}[1]
\Require Root node $\node$, depth limit $\Delta$
\State Initialize $\tree = \node$.
\While{$\tree$ contains an unfathomed leaf\label{step:mwhile_begin}}
	\State Select a leaf $Q$ of $\tree$ that maximizes $\nscore(\tree, Q)$.\label{step:mnsp}
	\If {$\texttt{depth}(\node) = \Delta$ or $\fathom(\tree, \node, \texttt{None},\ldots,\texttt{None})$} \label{step:mdepth_fathom}
	    \State Fathom $Q$.
	\Else
	    \State For $i = 1,\ldots, d$, select $A_i\in\actions_i(\tree, \node)$ that maximizes $\ascore_i(\tree,Q, A_i)$.\label{step:maction}
	    \If {$\fathom(\tree, \node, A_1,\ldots,A_d)$} \label{step:maction_fathom}
	        \State Fathom $Q$.
	    \ElsIf {$\children(\tree, \node, A_1,\ldots, A_d) = \emptyset$} \label{step:mno_children}
	        \State Fathom $Q$.
	    \Else 
	        \State Add all nodes in $\children(\tree, \node, A_1,\ldots,A_d)$ to $\tree$ as children of $Q$. \label{step:madd_children}
	    \EndIf
    \EndIf
\EndWhile\label{step:mwhile_end}
\end{algorithmic}
\end{algorithm}
There are two scoring rules $\ascore^i_1$ and $\ascore^i_2$ for each action type $i \in [d]$. Algorithm~\ref{alg:mTS} can then be parameterized by $(\vec{\mu}, \lambda)$, where $\vec{\mu}\in\R^d$ is a vector of parameters controlling each action: the $i$th action is selected to maximize $\mu_i\cdot\ascore^i_1 + (1-\mu_i)\cdot\ascore^i_2$. As before, we assume there are $b, k\in\N$ such that $|\actions_i(\tree, \node)|\le b$ for any $i$ and any $\node\in\tree$, and $|\children(\tree, \node, A_1,\ldots, A_d)|\le k$ for all $Q, A_1,\ldots, A_d$.

Let $\cA'$, as in the single-action setting, be Algorithm~\ref{alg:mTS} with the evaluations of $\fathom$ on line~\ref{step:mdepth_fathom} and line~\ref{step:maction_fathom} suppressed. Then, we may prove a slight generalization of lemma~\ref{lemma:a'}.

\begin{lemma}\label{lemma:ma'}
Let $\ascore^i_1$ and $\ascore^i_2$ be two path-wise action-selection scores, for each $i\in\{1,\ldots,d\}$. Fix the input root node $Q$. There are $T \leq k^{d\Delta(\Delta-1)/2}b^{d\Delta}$ boxes of the form $R_t = I_1\times\cdots\times I_d$ partitioning $[0,1]^d$ where for any box $R_t$, the action-selection scores $\mu_i\cdot\ascore^i_1 + (1-\mu_i)\cdot\ascore^i_2$ results in the same tree built by $\cA'$ for all $\vec{\mu}\in R_{t}$, no matter what node selection policy is used.
\end{lemma}

\begin{proof}
Let $\tree$ denote the tree built by $\cA'$. For $i \in [\Delta]$, let $\tree[i]$ denote the restriction of $\tree$ to nodes of depth at most $i$. Let $\ascore^i_{\mu_i} = \mu_i\cdot\ascore^i_1 + (1-\mu_i)\cdot\ascore^i_2$. We prove the lemma by induction on $i$. In particular, we show that for each $i\in[\Delta]$, there are $k^{di(i-1)/2}b^{di}$ boxes partitioning $[0,1]^d$ such that $\tree[i]$ is invariant over all $\vec{\mu}$ within any given box. Since $\tree[\Delta] = \tree$, this implies the lemma statement. The base case of $i=1$ is trivial since $\tree[1]$ consists of only the root, regardless of $\vec{\mu}\in[0,1]^d$.

Now, suppose the statement holds for some $i\in\{1,\ldots,\Delta-1\}$. That is, there are $T\le k^{di(i-1)/2}b^{di}$ disjoint boxes $R_1\cup\cdots\cup I_R = [0,1]^d$ such that $\tree[i]$ is invariant over all $\vec{\mu}$ within any given boxes (our inductive hypothesis). Fix one of these boxes $R_t$. We subdivide $R_t$ into sub-boxes such that $\tree[i+1]$ is invariant within each one of these smaller boxes. Let $Q$ be any leaf of $\tree[i]$, and for $\vec{\mu}\in R_t$ let $\tree_{\vec{\mu}}$ denote the state of the tree using $\ascore^i_{\mu_i}$ for each $i$ at the point that $Q$ is selected. Since $i < \Delta$, $Q$ is not fathomed at line~\ref{step:depth_fathom}, regardless of $\vec{\mu}$. Next, since $\actions_i$ is path-wise for each $i$, the actions available at $Q$ depend only on the path $\tree_Q$ from the root of $\tree$ to $Q$, which, by the inductive hypothesis, is invariant over all $\vec{\mu}\in R_t$. That is, for all $i$ $\actions_i(\tree_{\mu}, Q) = \actions_i(\tree_Q, Q)$ for all $\vec{\mu}\in R_t$. Then, $\ascore^i_{\mu_i}$ will select action $A_i\in\actions_i(\tree_Q, Q)$ if and only if \begin{align*}A_i &= \argmax_{A_0\in\actions_i(\tree_Q, \node)}\mu\cdot\ascore^i_1(\tree_{\vec{\mu}}, Q, A_0) + (1-\mu_i)\cdot\ascore^i_2(\tree_{\vec{\mu}}, Q, A_0) \\ &= \argmax_{A_0\in\actions_i(\tree_Q, \node)}\mu_i\cdot\ascore^i_1(\tree_Q, Q, A_0) + (1-\mu_i)\cdot\ascore^i_2(\tree_Q, Q, A_0),\end{align*} where the second equality follows from the fact that $\ascore^i_1$ and $\ascore^i_2$ are path-wise. Thus, for a fixed $A_0$, $\ascore^i_{\mu_i}$ is linear in $\mu_i$, so for each $A_0$ there is at most one subinterval of $[0,1]$ such that for all $\mu_i$ in that subinterval, $A_0$ maximizes $\ascore^i_{\mu_i}$. Thus, each leaf of $\tree[i]$ contributes at most $b$ subintervals such that for $\mu_i$ within a given subinterval, the action of type $i$ selected at each leaf of $\tree[i]$ is invariant. $\tree[i]$ consists of at most $k^i$ leaves, so this is a total of at most $k^ib$ subintervals. Writing $R_t = I_1\times\cdots I_d$, we have established that for each $i$, there are $k^ib$ subintervals partitioning $I_i$ into subintervals such that as $\mu_i$ varies over each subinterval, the action of type $i$ selected at every leaf of $\tree[i]$ is invariant. These subintervals partition $R_t$ into at most $(k^ib)^d$ boxes. As before, since the actions selected at each leaf $Q$ of $\tree[i]$ are invariant, the set of children $\children(\tree_{\vec{\mu}}, Q, A_1,\ldots, A_d) = \children(\tree_Q, Q, A_1,\ldots, A_d)$ of $Q$ added to the tree is also invariant, using the fact that $\children$ is path-wise. Therefore, within every sub-box of $R_t$, $\tree[i+1]$ is invariant. The total number of boxes over each possible $R_t$ is, by the induction hypothesis, at most $k^{di(i-1)/2}b^{di}\cdot k^{di}b^d =  k^{d(i+1)i/2}b^{d(i+1)}$.
\end{proof}

The proof of Lemma~\ref{lemma:rooted} is identical in the multi-action setting. The proof of Lemma~\ref{lemma:main} is also identical: here, we fix a box $R$ in the partition established in Lemma~\ref{lemma:ma'}, and get an identical partition of $R\times [0,1]$ such that the behavior of Algorithm~\ref{alg:mTS} is invariant as $\lambda$ varies in each subinterval of $[0,1]$. The number of boxes in the final partition of $[0,1]^{d+1}$ is $k^{d\Delta(\Delta-1)/2}b^{d\Delta}\cdot k^{5\Delta}\le k^{d\Delta(9+\Delta)}b^{d\Delta}.$ Our main pseudo-dimension bound for the multi-action setting follows from the same argument that exploits the framework of Balcan et al.~\cite{Balcan21:How}.

\begin{theorem}\label{theorem:multiple_actions_d}
Let $\cost(Q)$ be any tree-constant cost function, and let $\cost_{\vec{\mu},\lambda}(Q)$ be the cost of the tree built by Algorithm~\ref{alg:TS} on input root node $Q$ using action-selection scores parameterized by $\vec{\mu}\in\R^d$, where $d = O(1)$, and node-selection score parameterized by $\lambda$. Then, $\pdim(\{\cost_{\vec{\mu},\lambda}\}) = O(d\Delta^2\log k + d\Delta\log b)$.
\end{theorem}

When $d = O(1)$ we get the same pseudo-dimension bound as in the single-action setting: $\pdim(\{\cost_{\vec{\mu},\lambda}\}) = O(\Delta^2\log k + \Delta\log b)$, which is the statement of Theorem~\ref{theorem:multiple_actions}.

\end{document}